\documentclass[dvips,aos,preprint]{imsart}

\RequirePackage[OT1]{fontenc}
\RequirePackage{amsthm,amsmath,natbib}
\RequirePackage[colorlinks,citecolor=blue,urlcolor=blue]{hyperref}
\RequirePackage{hypernat}
\RequirePackage{amssymb}

\startlocaldefs
\numberwithin{equation}{section}
\theoremstyle{plain}

\newtheorem{lemma}{Lemma}[section]
\newtheorem{remark}{Remark}[section]
\newtheorem{theorem}{Theorem}[section]
\newtheorem{corollary}{Corollary}[section]
\endlocaldefs

\begin{document}

\def\ci{\!\perp\!}
\def\nci{\!\not\perp\!}

\begin{frontmatter}
\title{Faithfulness in Chain Graphs: The Gaussian Case}
\runtitle{Faithfulness in Chain Graphs: The Gaussian Case}

\begin{aug}
\author{\fnms{Jose M.} \snm{Pe\~{n}a}\ead[label=e1]{jospe@ida.liu.se}}

\runauthor{J. M. Pe\~{n}a}

\affiliation{ADIT, Department of Computer and Information Science\\
        Link\"oping University, SE-58183 Link\"{o}ping, Sweden\\
\printead{e1}\\
\phantom{E-mail: jospe@ida.liu.se}}

\end{aug}

\begin{abstract}
This paper deals with chain graphs under the classic Lauritzen-Wermuth-Frydenberg interpretation. We prove that the regular Gaussian distributions that factorize with respect to a chain graph $G$ with $d$ parameters have positive Lebesgue measure with respect to $\mathbb{R}^d$, whereas those that factorize with respect to $G$ but are not faithful to it have zero Lebesgue measure with respect to $\mathbb{R}^d$. This means that, in the measure-theoretic sense described, almost all the regular Gaussian distributions that factorize with respect to $G$ are faithful to it.
\end{abstract}



\end{frontmatter}

\section{Introduction}

This paper deals with chain graphs under the classic Lauritzen-Wermuth-Frydenberg interpretation. We prove that the regular Gaussian distributions that factorize with respect to a chain graph $G$ with $d$ parameters have positive Lebesgue measure with respect to $\mathbb{R}^d$, whereas those that factorize with respect to $G$ but are not faithful to it have zero Lebesgue measure with respect to $\mathbb{R}^d$. This means that, in the measure-theoretic sense described, almost all the regular Gaussian distributions that factorize with respect to $G$ are faithful to it. Previously, it has been proven that for any undirected graph there exists a regular Gaussian distribution that is faithful to it \citep[Corollary 3]{LnenickaandMatus2007}. A stronger result has been proven for acyclic directed graphs: In certain measure-theoretic sense, almost all the regular Gaussian distributions that factorize with respect to an acyclic directed graph are faithful to it \citep[Theorem 3.2]{Spirtesetal.1993}. Therefore, this paper extends the latter result to chain graphs. It is worth mentioning that we have recently proved in \citep{Penna2009} a result analogous to the one in this paper but for strictly positive discrete probability distributions with arbitrary prescribed sample space. It is also worth noticing that a result analogous to the one in this paper has been proven in \citep[Theorem 6.1]{Levitzetal.2001} under the alternative Andersson-Madigan-Perlman interpretation of chain graphs.

There are two important implications of the result proven in this paper:
\begin{itemize}
\item The use of chain graphs to represent independence models in artificial intelligence and statistics has increased over the years, particularly in the case of undirected graphs and acyclic directed graphs.\footnote{In this paper, we do not consider graphs with multiple edges between two nodes.} However, there are independence models that can be represented exactly by chain graphs but that cannot be represented exactly by undirected graphs or acyclic directed graphs. As a matter of fact, the experimental results in \citep{Penna2007} suggest that this may be the case for the vast majority of independence models that can be represented exactly by chain graphs. In other words, for most chain graphs, every undirected graph and acyclic directed graph either represents some separation statement that is not represented by the chain graph or does not represent some separation statement that is represented by the chain graph. As \citet[Section 1.1]{Studeny2005} points out, something that would confirm that this is an advantage of chain graphs for modeling regular Gaussian distributions would be proving that any independence model represented by a chain graph can be represented by a regular Gaussian distribution. The result in this paper confirms this point.
\item In the literature, there exist two graphical criteria for identifying independencies holding in a probability distribution $p$ that factorizes with respect to a chain graph $G$: The moralization criterion \citep{Lauritzen1996} and the c-separation criterion \citep{Studeny1998}. Both criteria are known to be equivalent \citep[Lemma 5.1]{Studeny1998}. Furthermore, both criteria are known to be sound, i.e. they only identify independencies in $p$ \citep[Theorems 3.34 and 3.36]{Lauritzen1996}. The result in this paper implies that both criteria are also complete for regular Gaussian distributions: If $p$ is a regular Gaussian distribution, then both criteria identify all the independencies in $p$ that can be identified on the sole basis of $G$, because there exists a regular Gaussian distribution that is faithful to $G$.
\end{itemize}

The rest of the paper is organized as follows. We start by reviewing some concepts in Section \ref{sec:preliminaries}. In Section \ref{sec:cg}, we describe how we parameterize the regular Gaussian distributions that factorize with respect to a chain graph. We present our results on faithfulness in Section \ref{sec:faithfulness}. In Section \ref{sec:equivalence}, we present some results about chain graph equivalence that follow from the results in Section \ref{sec:faithfulness}. Finally, we close with some discussion in Section \ref{sec:conclusions}.

\section{Preliminaries}\label{sec:preliminaries}

In this section, we define some concepts used later in this paper. We first recall some definitions from probabilistic graphical models. See, for instance, \citep{Lauritzen1996} and \citep{Studeny2005} for further information. Let $V=\{1, \ldots, N\}$ be a finite set of size $N$. The elements of $V$ are not distinguished from singletons and the union of the sets $I_1, \ldots, I_l \subseteq V$ is written as the juxtaposition $I_1 \ldots I_l$. We denote by $|I|$ the size or cardinality of a set $I \subseteq V$, e.g. $|V|=N$. We assume throughout the paper that the union of sets precedes the set difference when evaluating an expression. Unless otherwise stated, all the graphs in this paper are defined over $V$.

If a graph $G$ contains an undirected (resp. directed) edge between two nodes $v_{1}$ and $v_{2}$, then we write that $v_{1} - v_{2}$ (resp. $v_{1} \rightarrow v_{2}$) is in $G$. If $v_{1} \rightarrow v_{2}$ is in $G$ then $v_{1}$ is called a parent of $v_{2}$. Let $Pa_G(I)$ denote the set of parents in $G$ of the nodes in $I \subseteq V$. When $G$ is evident from the context, we drop the $G$ from $Pa_G(I)$ and use $Pa(I)$ instead. A route from a node $v_{1}$ to a node $v_{l}$ in a graph $G$ is a sequence of nodes $v_{1}, \ldots, v_{l}$ such that there exists an edge in $G$ between $v_{i}$ and $v_{i+1}$ for all $1 \leq i < l$. The length of a route is the number of (not necessarily distinct) edges in the route, e.g. the length of the route $v_{1}, \ldots, v_{l}$ is $l-1$. We treat all singletons as routes of length zero. A path is a route in which the nodes $v_{1}, \ldots, v_{l}$ are distinct. A route is called undirected if $v_{i} - v_{i+1}$ is in $G$ for all $1 \leq i < l$. A route is called descending if $v_{i} - v_{i+1}$ or $v_{i} \rightarrow v_{i+1}$ is in $G$ for all $1 \leq i < l$. If there is a descending route from $v_{1}$ to $v_{l}$ in $G$, then $v_{1}$ is called an ancestor of $v_{l}$ and $v_{l}$ is called a descendant of $v_{1}$. Let $An_G(I)$ denote the set of ancestors in $G$ of the nodes in $I \subseteq V$. A descending route $v_{1}, \ldots, v_{l}$ is called a directed pseudocycle if $v_{i} \rightarrow v_{i+1}$ is in $G$ for some $1 \leq i < l$, and $v_{l}=v_{1}$. A chain graph (CG) is a graph (possibly) containing both undirected and directed edges and no directed pseudocycles. An undirected graph (UG) is a CG containing only undirected edges. The underlying UG of a CG is the UG resulting from replacing the directed edges in the CG by undirected edges. A set of nodes of a CG is connected if there exists an undirected route in the CG between every pair of nodes in the set. A connectivity component of a CG is a connected set that is maximal with respect to set inclusion. Hereinafter, we assume that the connectivity components $B_1, \ldots, B_n$ of a CG $G$ are well-ordered, i.e. if $v_{1} \rightarrow v_{2}$ is in $G$ then $v_{1} \in B_i$ and $v_{2} \in B_j$ for some $1 \leq i < j \leq n$. The moral graph of a CG $G$, denoted $G^m$, is the undirected graph where two nodes are adjacent iff they are adjacent in $G$ or they are both in $Pa(B_i)$ for some connectivity component $B_i$ of $G$. The subgraph of $G$ induced by $I \subseteq V$, denoted $G_I$, is the graph over $I$ where two nodes are connected by a (un)directed edge if that edge is in $G$. A path $v_{1}, \ldots, v_{l}$ in $G$ is called a complex if the subgraph of $G$ induced by the set of nodes in the path looks like $v_{1} \rightarrow v_{2} - \ldots - v_{l-1} \leftarrow v_{l}$. The path $v_{2}, \ldots, v_{l-1}$ is called the region of the complex. A section of a route $\rho$ in a CG is a maximal subroute of $\rho$ that only contains undirected edges. A section $v_{2} - \ldots - v_{l-1}$ of $\rho$ is a collider section of $\rho$ if $v_{1} \rightarrow v_{2} - \ldots - v_{l-1} \leftarrow v_{l}$ is a subroute of $\rho$. Furthermore, a route $\rho$ in a CG is said to be superactive with respect to $K \subseteq V$ when
\begin{itemize}
\item every collider section of $\rho$ has some node in $K$, and
\item every other section of $\rho$ has no node in $K$.
\end{itemize}
A set $I \subseteq V$ is complete in an UG $G$ if there is an undirected edge in $G$ between every pair of distinct nodes in $I$. We denote the set of complete sets in $G$ by $\mathcal{C}(G)$. We treat all singletons as complete sets and, thus, they are included in $\mathcal{C}(G)$.

Let $X=(X_i)_{i \in V}$ denote a column vector of random variables and $X_I$ $(I \subseteq V)$ its subvector $(X_i)_{i \in I}$. We use upper-case letters to denote random variables and the same letters in lower-case to denote their states. Unless otherwise stated, all the probability distributions in this paper are defined on (state space) $\mathbb{R}^N$. Let $I$, $J$ and $K$ denote three disjoint subsets of $V$. We denote by $I \ci_p J | K$ that $X_I$ is independent of $X_J$ given $X_K$ in a probability distribution $p$. Likewise, we denote by $I \ci_G J | K$ that $I$ is separated from $J$ given $K$ in a CG $G$. Specifically, $I \ci_G J | K$ holds when there is no route in $G$ from a node in $I$ to a node in $J$ that is superactive with respect to $K$. This is equivalent to say that $I \ci_G J | K$ holds when every path in $(G_{An_G(IJK)})^m$ from a node in $I$ to a node in $J$ has some node in $K$. The independence model represented by a CG $G$ is the set of separation statements $I \ci_G J | K$. We say that a probability distribution $p$ is Markovian with respect to a CG $G$ when $I \ci_p J | K$ if $I \ci_G J | K$ for all $I$, $J$ and $K$ disjoint subsets of $V$. We say that $p$ is faithful to $G$ when $I \ci_p J | K$ iff $I \ci_G J | K$ for all $I$, $J$ and $K$ disjoint subsets of $V$. We denote by $I \nci_p J | K$ and $I \nci_G J | K$ that $I \ci_p J | K$ and $I \ci_G J | K$ do not hold, respectively.

We now recall some results from matrix theory. See, for instance, \citep{HornandJohnson1985} for more information. Let $A=(A_{i,j})_{i, j \in V}$ denote a square matrix. Let $A_{I,J}$ with $I, J \subseteq V$ denote its submatrix $(A_{i,j})_{i \in I, j \in J}$. The determinant of $A$ can recursively be computed, for fixed $i \in V$, as $det(A)=\sum_{j \in V} (-1)^{i+j} A_{i,j} det(A_{\setminus (ij)})$, where $A_{\setminus (ij)}$ denotes the matrix produced by removing the row $i$ and column $j$ from $A$. If $det(A) \neq 0$ then the inverse of $A$ can be computed as $(A^{-1})_{i,j}=(-1)^{i+j} det(A_{\setminus (ji)})/det(A)$ for all $i, j \in V$. We say that $A$ is strictly diagonally dominant if $abs(A_{i,i}) > \sum_{\{j \in V \: : \: j \neq i\}} abs(A_{i,j})$ for all $i \in V$, where $abs()$ denotes absolute value. A matrix $A$ is Hermitian if it is equal to the matrix resulting from, first, transposing $A$ and, then, replacing each entry by its complex conjugate. Clearly, a real symmetric matrix is Hermitian. A real symmetric $N \times N$ matrix $A$ is positive definite if $y^T A y >0$ for all non-zero $y \in \mathbb{R}^N$.

\begin{remark}\label{rem:inverse}
Note that $det(A)$ is a real polynomial in the entries of $A$, and that $(A^{-1})_{i,j}$ is then the restriction of a fraction of two real polynomials in the entries of $A$ to the area where $det(A)$ is non-zero.
\end{remark}

Finally, we recall some results about Gaussian distributions. We represent a Gaussian distribution as $\mathcal{N}(\mu, \Sigma)$ where $\mu$ is its mean vector and $\Sigma$ its covariance matrix. We say that a Gaussian distribution $\mathcal{N}(\mu, \Sigma)$ is regular if $\Sigma$ is positive definite or, equivalently, invertible. In this paper, we often find more convenient to work with the inverse of the covariance matrix $\Omega=\Sigma^{-1}$, which is also known as the concentration matrix or precision matrix. Since $\Sigma=\Omega^{-1}$, we thus often write $\mathcal{N}(\mu, \Omega^{-1})$ instead of $\mathcal{N}(\mu, \Sigma)$. Let $I$, $J$, $K$ and $L$ denote four disjoint subsets of $V$. Any regular Gaussian distribution $p$ satisfies, among others, the following properties:
\begin{itemize}
\item Symmetry $I \ci_p J | K \Rightarrow J \ci_p I | K$.
\item Decomposition $I \ci_p J L | K \Rightarrow I \ci_p J | K$.
\item Intersection $I \ci_p J | K L \land I \ci_p L | K J \Rightarrow I \ci_p J L | K$.
\item Weak transitivity $I \ci_p J | K \land I \ci_p J | K u \Rightarrow I \ci_p u | K \lor u \ci_p J | K$ with $u \in V \setminus I J K$.
\end{itemize}

The following results have been proven in \citep[Sections 2.3.1, 2.3.3]{Bishop2006}. For the sake of completeness, Appendix A adapts the proofs to the notation used in this paper. Let $I$ and $J$ denote two disjoint subsets of $V$. Let $p(x_{IJ}) = \mathcal{N}(\mu, \Omega^{-1})$ where $\Omega$ is positive definite. Then, as shown in \citep[Section 2.3.1]{Bishop2006} and in Appendix A, $p(x_J | x_I) = \mathcal{N}(\delta x_I + \gamma, \epsilon^{-1})$ where $\delta$, $\gamma$ and $\epsilon$ are the following real matrices of dimensions, respectively, $|J| \times |I|$, $|J| \times 1$ and $|J| \times |J|$:
\begin{equation}\label{eq:bishop1a}
\delta=-(\Omega_{J,J})^{-1} \Omega_{J,I},
\end{equation}
\begin{equation}\label{eq:bishop1b}
\gamma=\mu_J+(\Omega_{J,J})^{-1} \Omega_{J,I} \mu_I
\end{equation}
and
\begin{equation}\label{eq:bishop1c}
\epsilon=\Omega_{J,J}.
\end{equation}
Let $p(x_I) = \mathcal{N}(\alpha, \beta^{-1})$ and $q(x_J | x_I) = \mathcal{N}(\delta x_I + \gamma, \epsilon^{-1})$ where $\delta$, $\gamma$ and $\epsilon$ are real matrices of dimensions, respectively, $|J| \times |I|$, $|J| \times 1$ and $|J| \times |J|$, and $\beta$ and $\epsilon$ are positive definite. Then, as shown in \citep[Section 2.3.3]{Bishop2006} and in Appendix A, $p(x_I) q(x_J | x_I)$ is a Gaussian distribution $\mathcal{N}(\lambda, \Lambda^{-1})$ over $\left(
\begin{array}{c}
    x_I\\
    x_J\\
  \end{array}
\right)$
where
\begin{equation}\label{eq:bishop2a}
\lambda=\left(
  \begin{array}{c}
    \alpha \\
    \delta \alpha + \gamma \\
  \end{array}
\right)
\end{equation}
and
\begin{equation}\label{eq:bishop2b}
\Lambda = \left(
  \begin{array}{cc}
    \beta + \delta^T \epsilon \delta & - \delta^T \epsilon \\
    -\epsilon \delta & \epsilon \\
  \end{array}
\right).
\end{equation}
Moreover, $p(x_I) q(x_J | x_I)$ is regular because
\begin{equation}\label{eq:bishop2c}
\Lambda^{-1} = \left(
  \begin{array}{cc}
    \beta^{-1} & \beta^{-1} \delta^T \\
    \delta \beta^{-1} & \epsilon^{-1} + \delta \beta^{-1} \delta^T \\
  \end{array}
\right).
\end{equation}

\section{Parameterization of chain graphs}\label{sec:cg}

In this section, we describe how we parameterize the regular Gaussian distributions that factorize with respect to a CG. This is a key issue because our results about faithfulness are not only relative to the CG at hand and the measure considered, the Lebesgue measure, but also to the number of parameters of the regular Gaussian distributions that factorize with respect to the CG at hand.

We say that a regular Gaussian distribution $p$ factorizes with respect to a CG $G$ with connectivity components $B_1, \ldots, B_n$ if the following two conditions are met \citep[Proposition 3.30]{Lauritzen1996}:
\begin{itemize}
\item[F1.] $p(x) = \prod_{i=1}^n p(x_{B_i} | x_{Pa(B_i)})$ where
\item[F2.] $p(x_{B_i Pa(B_i)}) = \prod_{C \in \mathcal{C}((G_{B_i Pa(B_i)})^m)} \psi^i_C(x_C)$ where each $\psi^i_C(x_C)$ is a non-negative real function.
\end{itemize}
Let $\mathcal{N}(G)$ denote the set of regular Gaussian distributions that factorize with respect to $G$. We parameterize each probability distribution $p \in \mathcal{N}(G)$ with the following parameters:
\begin{itemize}
\item The mean vector $\mu$ of $p$.
\item The submatrices $\Omega^i_{B_i,B_i}$ and $\Omega^i_{B_i ,Pa(B_i)}$ of the precision matrix $\Omega^i$ of $p(x_{B_i Pa(B_i)})$ for all $1 \leq i \leq n$.
\end{itemize}
We warn the reader that if $\Omega$ denotes the precision matrix of $p$, then $\Omega^i$ is not $\Omega_{B_i Pa(B_i), B_i Pa(B_i)}$ but $((\Omega^{-1})_{B_i Pa(B_i), B_i Pa(B_i)})^{-1}$. It is worth mentioning that an alternative parameterization of the probability distributions in $\mathcal{N}(G)$ is presented in \citep{Wermuth1992}. The main difference between our parameterization and the alternative one is that we parameterize certain concentration matrices whereas they parameterize certain partial concentration matrices. However, both parameterizations are equivalent. We omit the details of the equivalence because they are irrelevant for our purpose. We stick to our parameterization simply because it is more convenient for the calculations performed later in this paper.

Note that the values of some of the parameters in the parameterization introduced above are determined by the values of the rest of the parameters. Specifically, for all $1 \leq i \leq n$, the following constraints apply:
\begin{itemize}
\item[C1.] $(\Omega^i_{B_i,B_i})_{j,k}=(\Omega^i_{B_i,B_i})_{k,j}$ for all $j, k \in B_i$, because $\Omega^i_{j,k}=\Omega^i_{k,j}$ since $\Omega^i$ is symmetric.
\item[C2.] $(\Omega^i_{B_i,B_i})_{j,k}=0$ for all $j, k \in B_i$ such that $j$ and $k$ are not adjacent in $G$. To see it, note that $j$ and $k$ are not adjacent in $(G_{B_i Pa(B_i)})^m$. Consequently, any path between $j$ and $k$ in $(G_{B_i Pa(B_i)})^m$ must pass through some node in $B_i \setminus j k$ or $Pa(B_i)$. Then, $j \ci_{(G_{B_i Pa(B_i)})^m} k | B_i Pa(B_i) \setminus j k$, which implies $j \ci_{p(x_{B_i Pa(B_i)})} k | B_i Pa(B_i) \setminus j k$ because $p(x_{B_i Pa(B_i)})$ is Markovian with respect to $(G_{B_i Pa(B_i)})^m$ due to the condition F2 above \citep[Proposition 3.30, Theorems 3.34 and 3.36]{Lauritzen1996}. The latter independence statement implies $\Omega^i_{j,k}=0$ and, thus, $(\Omega^i_{B_i,B_i})_{j,k}=0$ \citep[Proposition 5.2]{Lauritzen1996}.
\item[C3.] $(\Omega^i_{B_i ,Pa(B_i)})_{j,k}=0$ for all $j \in B_i$ and $k \in Pa(B_i)$ such that $j$ and $k$ are not adjacent in $G$, by a reasoning analogous to the one above.
\end{itemize}
Hereinafter, the parameters whose values are not determined by the constraints above are called non-determined (nd) parameters. However, the values the nd parameters can take are constrained by the fact that these values must correspond to some probability distribution in $\mathcal{N}(G)$. We prove in Lemma \ref{lem:121cg} that this is equivalent to requiring that the nd parameters can only take real values such that $\Omega^i_{B_i,B_i}$ is positive definite for all $1 \leq i \leq n$. That is why the set of nd parameter values satisfying this requirement are hereinafter called the nd parameter space for $\mathcal{N}(G)$. We do not work out the inequalities defining the nd parameter space because these are irrelevant for our purpose. The number of nd parameters is what we call the dimension of $G$, and we denote it as $d$. Specifically, $d=2 |V| + |G|$ where $|G|$ is the number of edges in $G$:
\begin{itemize}
\item $|V|$ due to $\mu$.
\item $|V|$ due to $(\Omega^i_{B_i,B_i})_{j,j}$ for all $1 \leq i \leq n$ and $j \in B_i$.
\item $|G|$ due to the entries below the diagonal of $\Omega^i_{B_i,B_i}$ that are not identically zero and the entries of $\Omega^i_{B_i ,Pa(B_i)}$ that are not identically zero for all $1 \leq i \leq n$. To see this, recall from the constraints C1-C3 above that there is one entry below the diagonal in some $\Omega^i_{B_i,B_i}$ that is not identically zero for each undirected edge in $G$, and one entry in some $\Omega^i_{B_i ,Pa(B_i)}$ that is not identically zero for each directed edge in $G$.
\end{itemize}

\begin{lemma}\label{lem:121cg}
Let $G$ be a CG. There is a one-to-one correspondence between the probability distributions in $\mathcal{N}(G)$ and the elements of the nd parameter space for $\mathcal{N}(G)$.
\end{lemma}

\begin{proof}
We first prove that the mapping of probability distributions into nd parameter values is injective. Obviously, any probability distribution in $p \in \mathcal{N}(G)$ is mapped into some real values of the nd parameters $\mu$, $\Omega^i_{B_i,B_i}$ and $\Omega^i_{B_i ,Pa(B_i)}$ for all $1 \leq i \leq n$. In particular, $\Omega^i_{B_i,B_i}$ takes value $(((\Omega^{-1})_{B_i Pa(B_i), B_i Pa(B_i)})^{-1})_{B_i,B_i}$ where $\Omega$ is the precision matrix of $p$. Then, that $\Omega^i_{B_i,B_i}$ is positive definite follows from the fact that $\Omega$ is positive definite \citep[p. 237]{Studeny2005}. Thus, $p$ is mapped into some element of the nd parameter space for $\mathcal{N}(G)$.

Moreover, different probability distributions are mapped into different elements. To see it, assume to the contrary that there exist two distinct probability distributions $p, p' \in \mathcal{N}(G)$ that are mapped into the same element. Note that this element uniquely identifies $p(x_{B_i} | x_{Pa(B_i)})$ by Equations \ref{eq:bishop1a}-\ref{eq:bishop1c} for all $1 \leq i \leq n$, where $I=Pa(B_i)$ and $J=B_i$. Likewise, it uniquely identifies $p'(x_{B_i} | x_{Pa(B_i)})$ for all $1 \leq i \leq n$. Then, $p(x_{B_i} | x_{Pa(B_i)}) = p'(x_{B_i} | x_{Pa(B_i)})$ for all $1 \leq i \leq n$. However, this contradicts the assumption that $p$ and $p'$ are distinct by the condition F1 above.

We now prove in three steps that the mapping of nd parameter values into probability distributions is injective.

\textbf{Step 1} We first show that any element of the nd parameter space for $\mathcal{N}(G)$ is mapped into some regular Gaussian distribution $q$. Note that any element of the nd parameter space for $\mathcal{N}(G)$ uniquely identifies a Gaussian distribution $q^i(x_{B_i} | x_{Pa(B_i)})$ for all $1 \leq i \leq n$ by Equations \ref{eq:bishop1a}-\ref{eq:bishop1c}, where $I=Pa(B_i)$ and $J=B_i$. Specifically, $q^i(x_{B_i} | x_{Pa(B_i)})= \mathcal{N}(\delta^i x_{Pa(B_i)} + \gamma^i, (\epsilon^i)^{-1})$ where
\begin{equation}\label{eq:delta}
\delta^i= -(\Omega^i_{B_i,B_i})^{-1} \Omega^i_{B_i ,Pa(B_i)},
\end{equation}
\begin{equation}\label{eq:gamma}
\gamma^i= \mu_{B_i} + (\Omega^i_{B_i,B_i})^{-1} \Omega^i_{B_i ,Pa(B_i)} \mu_{Pa(B_i)}
\end{equation}
and
\begin{equation}\label{eq:epsilon}
\epsilon^i= \Omega^i_{B_i,B_i}.
\end{equation}
In the equations above, we have assumed that the values of all the entries of $\Omega^i_{B_i,B_i}$ and $\Omega^i_{B_i ,Pa(B_i)}$ have previously been determined from the element of the nd parameter space at hand and the constraints C1-C3 above. Furthermore, note that $q^i(x_{B_i} | x_{Pa(B_i)})$ is regular because, by definition, $\Omega^i_{B_i,B_i}$ is positive definite. Clearly, $q^i(x_{B_i} | x_{Pa(B_i)})$ can be rewritten as a regular Gaussian distribution $r^i(x_{B_i} | x_{B_1 \ldots B_{i-1}})$: It suffices to take
\[
r^i(x_{B_i} | x_{B_1 \ldots B_{i-1}}) = \mathcal{N}((\delta^i, \textbf{0}) \left(
  \begin{array}{c}
    x_{Pa(B_i)}\\
    x_{B_1 \ldots B_{i-1} \setminus Pa(B_i)}\\
  \end{array}
\right) + \gamma^i, (\epsilon^i)^{-1})
\]
where $\textbf{0}$ is a matrix of zeroes of dimension $|B_i| \times |B_1 \ldots B_{i-1} \setminus Pa(B_i)|$. Then, $r^1(x_{B_1}) r^2(x_{B_2} | x_{B_1})$ is a regular Gaussian distribution by Equations \ref{eq:bishop2a}-\ref{eq:bishop2c}. Likewise, $r^1(x_{B_1}) r^2(x_{B_2} | x_{B_1}) r^3(x_{B_3} | x_{B_1 B_2})$ is a regular Gaussian distribution. Continuing with this process for the rest of connectivity components proves that $\prod_{i=1}^n q^i(x_{B_i} | x_{Pa(B_i)})=\prod_{i=1}^n r^i(x_{B_i} | x_{B_1 \ldots B_{i-1}})$ is mapped into some regular Gaussian distribution $q$.

\textbf{Step 2} We now show that $q \in \mathcal{N}(G)$. Note that for all $1 \leq i < n$ and any fixed value of $x_{B_1 \ldots B_i}$
\[
\int \prod_{l=i+1}^n q^l(x_{B_l} | x_{Pa(B_l)}) d{x_{B_{i+1} \ldots B_n}}
\]
\[
= \int q^{i+1}(x_{B_{i+1}} | x_{Pa(B_{i+1})}) [ \int q^{i+2}(x_{B_{i+2}} | x_{Pa(B_{i+2})}) [ \ldots
\]
\[
\ldots [ \int q^n(x_{B_n} | x_{Pa(B_n)}) d{x_{B_n}}] \ldots ]d{x_{B_{i+2}}}]d{x_{B_{i+1}}} = 1.
\]
Thus, for all $1 \leq i \leq n$, it follows from the equation above that
\[
q(x_{B_i Pa(B_i)}) = \int \prod_{l=1}^n q^l(x_{B_l} | x_{Pa(B_l)}) d{x_{B_1 \ldots B_n \setminus B_i Pa(B_i)}}
\]
\[
= \int [ \prod_{l=1}^i q^l(x_{B_l} | x_{Pa(B_l)}) ] [ \int \prod_{l=i+1}^n q^l(x_{B_l} | x_{Pa(B_l)}) d{x_{B_{i+1} \ldots B_n}}] d{x_{B_1 \ldots B_{i-1} \setminus Pa(B_i)}}
\]
\[
= \int \prod_{l=1}^i q^l(x_{B_l} | x_{Pa(B_l)}) d{x_{B_1 \ldots B_{i-1} \setminus Pa(B_i)}}
\]
\begin{equation}\label{eq:sum1}
= q^i(x_{B_i} | x_{Pa(B_i)}) \int \prod_{l=1}^{i-1} q^l(x_{B_l} | x_{Pa(B_l)}) d{x_{B_1 \ldots B_{i-1} \setminus Pa(B_i)}}.
\end{equation}
Moreover, for all $1 \leq i \leq n$, it follows from the equation above that
\[
q(x_{Pa(B_i)}) = \int q(x_{B_i Pa(B_i)}) d{x_{B_i}}
\]
\[
= \int [ \int \prod_{l=1}^i q^l(x_{B_l} | x_{Pa(B_l)}) d{x_{B_1 \ldots B_{i-1} \setminus Pa(B_i)}} ] d{x_{B_i}}
\]
\begin{equation}\label{eq:fub1}
= \int [ \int \prod_{l=1}^i q^l(x_{B_l} | x_{Pa(B_l)}) d{x_{B_i}} ] d{x_{B_1 \ldots B_{i-1} \setminus Pa(B_i)}}
\end{equation}
\begin{equation}\label{eq:fub2}
= \int [ \prod_{l=1}^{i-1} q^l(x_{B_l} | x_{Pa(B_l)}) \int q^i(x_{B_i} | x_{Pa(B_i)}) d{x_{B_i}} ] d{x_{B_1 \ldots B_{i-1} \setminus Pa(B_i)}}
\end{equation}
\begin{equation}\label{eq:sum2}
= \int \prod_{l=1}^{i-1} q^l(x_{B_l} | x_{Pa(B_l)}) d{x_{B_1 \ldots B_{i-1} \setminus Pa(B_i)}}.
\end{equation}
Note the use of Fubini's theorem to change the order of integration and produce Equation \ref{eq:fub1}. This implies that the inner integral in Equation \ref{eq:fub2} becomes 1. Consequently, for all $1 \leq i \leq n$
\begin{equation}\label{eq:equalconditionals}
q(x_{B_i} | x_{Pa(B_i)}) = \frac{q(x_{B_i Pa(B_i)})}{q(x_{Pa(B_i)})} = q^i(x_{B_i} | x_{Pa(B_i)})
\end{equation}
due to Equations \ref{eq:sum1} and \ref{eq:sum2}. Therefore,
\[
q(x) = \prod_{i=1}^n q^i(x_{B_i} | x_{Pa(B_i)}) = \prod_{i=1}^n q(x_{B_i} | x_{Pa(B_i)})
\]
and, thus, $q$ satisfies the condition F1 above. Moreover, $q(x_{B_i Pa(B_i)})$ satisfies the condition F2 for all $1 \leq i \leq n$. We show this by induction on $i$. Let $\Lambda^i$ denote the precision matrix of $q(x_{B_i Pa(B_i)})$, and note that
\[
q(x_{B_i Pa(B_i)})= q^i(x_{B_i} | x_{Pa(B_i)}) q(x_{Pa(B_i)})
\]
by Equation \ref{eq:equalconditionals}. So, $\Lambda^i$ can be calculated from $q^i(x_{B_i} | x_{Pa(B_i)})$ and $q(x_{Pa(B_i)})$ via Equation \ref{eq:bishop2b}. Specifically, it follows from Equations \ref{eq:bishop2b} and \ref{eq:epsilon}, respectively \ref{eq:delta}, that
\begin{equation}\label{eq:lambda1}
\Lambda^i_{B_i,B_i} = \epsilon^i = \Omega^i_{B_i,B_i}
\end{equation}
and
\begin{equation}\label{eq:lambda2}
\Lambda^i_{B_i ,Pa(B_i)} = -\epsilon^i \delta^i =- \Omega^i_{B_i,B_i} [-(\Omega^i_{B_i,B_i})^{-1} \Omega^i_{B_i ,Pa(B_i)}]=\Omega^i_{B_i ,Pa(B_i)}.
\end{equation}
Consequently, due to the constraints C2 and C3 above, $\Lambda^i_{j,k}=0$ for all $j, k \in B_i Pa(B_i)$ such that $j$ and $k$ are not adjacent in $(G_{B_i Pa(B_i)})^m$. Moreover, $\Lambda^i_{j,k}=0$ is equivalent to $j \ci_{q(x_{B_i Pa(B_i)})} k | B_i Pa(B_i) \setminus j k$ \citep[Proposition 5.2]{Lauritzen1996}. This implies that $q(x_{B_i Pa(B_i)})$ factorizes with respect to $(G_{B_i Pa(B_i)})^m$ and, thus, that it satisfies the condition F2 above \citep[Proposition 3.30, Theorems 3.34 and 3.36]{Lauritzen1996}. Consequently, $q \in \mathcal{N}(G)$.

\textbf{Step 3} We finally show that different elements of the nd parameter space for $\mathcal{N}(G)$ are mapped into different probability distributions in $\mathcal{N}(G)$. Assume to the contrary that two distinct elements of the nd parameter space for $\mathcal{N}(G)$ are mapped into the same probability distribution $q \in \mathcal{N}(G)$. Assume that the two elements differ in the value for $\mu_{B_i}$, $\Omega^i_{B_i,B_i}$ or $\Omega^i_{B_i ,Pa(B_i)}$ but that they coincide in the values for $\mu_{B_l}$, $\Omega^l_{B_l,B_l}$ and $\Omega^l_{B_l ,Pa(B_l)}$ for all $1 \leq l < i$. There are two scenarios to consider:
\begin{itemize}
\item If the two elements differ in the value for $\Omega^i_{B_i,B_i}$ or $\Omega^i_{B_i ,Pa(B_i)}$, then they are mapped into two different $q(x_{B_i Pa(B_i)})$ by Equations \ref{eq:lambda1} and \ref{eq:lambda2}, because two regular Gaussian distributions with different precision matrices are different. However, this contradicts the assumption that the two elements are mapped into the same $q$.
\item If the two elements differ in the value for $\mu_{B_i}$ but they do not differ in the values for $\Omega^i_{B_i,B_i}$ and $\Omega^i_{B_i ,Pa(B_i)}$, then the two elements do not differ in the value for $\mu_{Pa(B_i)}$ either, because $Pa(B_i) \subseteq B_1 \ldots B_{i-1}$ and we assumed above that the two elements coincide in the values for $\mu_{B_l}$ for all $1 \leq l < i$. Then, the two elements are mapped into the same $\delta^i$ but different $\gamma^i$ in Equations \ref{eq:delta} and \ref{eq:gamma}. That is, the two elements are mapped into two different $q^i(x_{B_i} | x_{Pa(B_i)})$ and, thus, to two different $q(x_{B_i} | x_{Pa(B_i)})$ by Equation \ref{eq:equalconditionals}. However, this contradicts the assumption that the two elements are mapped into the same $q$.
\end{itemize}
\end{proof}

\begin{remark}\label{rem:polynomial}
Note the following three observations:
\begin{itemize}
\item For all $1 \leq i \leq n$, according to the constraints C1-C3 above, every entry of $\Omega^i_{B_i,B_i}$ and $\Omega^i_{B_i ,Pa(B_i)}$ is equal either to zero or to some nd parameter in the parameterization of the probability distributions in $\mathcal{N}(G)$.
\item For all $1 \leq i \leq n$, by Remark \ref{rem:inverse}, every entry of $(\Omega^i_{B_i,B_i})^{-1}$ is a fraction of real polynomials in the entries of $\Omega^i_{B_i,B_i}$ and, thus, a fraction of real polynomials in the nd parameters in the parameterization of the probability distributions in $\mathcal{N}(G)$. Thus, every entry of the matrices $\delta^i$ and $\epsilon^i$ in Equations \ref{eq:delta} and \ref{eq:epsilon} is also a fraction of real polynomials in the referred nd parameters.
\item Every entry of the precision matrix of $r^1(x_{B_1}) r^2(x_{B_2} | x_{B_1})$ in the proof above is, by Equation \ref{eq:bishop2b}, a real polynomial in the entries of $\delta^2$, $\epsilon^2$ and the precision matrix of $r^1(x_{B_1})$, i.e. $\epsilon^1$. Likewise, every entry of the precision matrix of $r^1(x_{B_1}) r^2(x_{B_2} | x_{B_1}) r^3(x_{B_3} | x_{B_1 B_2})$ in the proof above is a real polynomial in the entries of $\delta^3$, $\epsilon^3$ and the precision matrix of $r^1(x_{B_1}) r^2(x_{B_2} | x_{B_1})$, that is, a real polynomial in the entries of $\delta^3$, $\epsilon^3$, $\delta^2$, $\epsilon^2$ and $\epsilon^1$. Continuing with this process for the rest of connectivity components shows that every entry of the precision matrix of $q(x) = \prod_{i=1}^n r^i(x_{B_i} | x_{B_1 \ldots B_{i-1}})$ in the proof above is a real polynomial in the entries of the matrices $\epsilon^1$, and $\delta^i$ and $\epsilon^i$ for all $1 < i \leq n$.
\end{itemize}

It follows from the observations above that every entry of the precision matrix of $q$ in the proof above is a fraction of real polynomials in the nd parameters in the parameterization of the probability distributions in $\mathcal{N}(G)$. Consequently, by Remark \ref{rem:inverse}, every entry of the covariance matrix of $q$ is a fraction of real polynomials in the nd parameters in the parameterization of the probability distributions in $\mathcal{N}(G)$. Moreover, note the following two observations on the latter fractions:
\begin{itemize}
\item Each of these fractions is defined on the whole nd parameter space for $\mathcal{N}(G)$: The polynomial in the denominator of the fraction is non-vanishing in the nd parameter space for $\mathcal{N}(G)$ because, as we have proven in Step 1 in the theorem above, $q$ is a Gaussian distribution.
\item Within the nd parameter space for $\mathcal{N}(G)$, each of these fractions vanishes only in the points where the polynomial in the numerator of the fraction vanishes because, as we have just seen, the denominator of the fraction is non-vanishing in the nd parameter space for $\mathcal{N}(G)$.
\end{itemize}
\end{remark}

We now prove another result that will be crucial in the coming section.

\begin{lemma}\label{lem:spacecg}
Let $G$ be a CG of dimension $d$. The nd parameter space for $\mathcal{N}(G)$ has positive Lebesgue measure with respect to $\mathbb{R}^d$.
\end{lemma}

\begin{proof}
Since we do not know a closed-form expression of the nd parameter space for $\mathcal{N}(G)$, we take an indirect approach to prove the lemma. Recall that, by definition, the nd parameter space for $\mathcal{N}(G)$ is the set of real values such that, after the extension determined by the constraints C1 and C2, $\Omega^i_{B_i,B_i}$ is positive definite for all $1 \leq i \leq n$. Therefore, all the nd parameters except those in $\Omega^i_{B_i,B_i}$ for all $1 \leq i \leq n$ can take values independently of the rest of the nd parameters. The nd parameters in $\Omega^i_{B_i,B_i}$ cannot take values independently one of another because, otherwise, $\Omega^i_{B_i,B_i}$ may not be positive definite. However, if the entries in the diagonal of $\Omega^i_{B_i,B_i}$ take values in $(|B_i|-1, \infty)$ and the rest of the nd parameters in $\Omega^i_{B_i,B_i}$ take values in $[-1,1]$, then the nd parameters in $\Omega^i_{B_i,B_i}$ can take values independently one of another. To see it, note that in this case $\Omega^i_{B_i,B_i}$ will always be Hermitian, strictly diagonally dominant, and with strictly positive diagonal entries, which implies that $\Omega^i_{B_i,B_i}$ will always be positive definite \citep[Corollary 7.2.3]{HornandJohnson1985}.

The subset of the nd parameter space of $\mathcal{N}(G)$ described in the paragraph above has positive volume in $\mathbb{R}^d$ and, thus, it has positive Lebesgue measure with respect to $\mathbb{R}^d$. Then, the nd parameter space of $\mathcal{N}(G)$ has positive Lebesgue measure with respect to $\mathbb{R}^d$.
\end{proof}

\section{Faithfulness in chain graphs}\label{sec:faithfulness}

The two theorems below are the main contribution of this manuscript. They prove that for any CG $G$, in the measure-theoretic sense described below, almost all the probability distributions in $\mathcal{N}(G)$ are faithful to $G$.

\begin{theorem}\label{the:f1}
Let $G$ be a CG of dimension $d$. $\mathcal{N}(G)$ has positive Lebesgue measure with respect to $\mathbb{R}^d$.
\end{theorem}

\begin{proof}
The one-to-one correspondence proved in Lemma \ref{lem:121cg} enables us to compute the Lebesgue measure with respect to $\mathbb{R}^d$ of $\mathcal{N}(G)$ as the Lebesgue measure with respect to $\mathbb{R}^d$ of the nd parameter space for $\mathcal{N}(G)$. Moreover, the latter is positive by Lemma \ref{lem:spacecg}.
\end{proof}

Before proving the second theorem, some auxiliary lemmas are proven.

\begin{lemma}\label{lem:subset}
Let $G$ and $H$ be two CGs such that the undirected (resp. directed) edges in $H$ are a subset of the undirected (resp. directed) edges in $G$. Then, $\mathcal{N}(H) \subseteq \mathcal{N}(G)$.
\end{lemma}

\begin{proof}
Note that a regular Gaussian distribution factorizes with respect to a CG iff it is Markovian with respect to the CG \citep[Proposition 3.30, Theorems 3.34 and 3.36]{Lauritzen1996}. Then, $\mathcal{N}(H) \subseteq \mathcal{N}(G)$ because the independence model represented by $H$ is a superset of that represented by $G$.
\end{proof}

\begin{lemma}\label{lem:chunk1}
Let $G$ be a CG such that
\begin{enumerate}
\item $G$ has a route between the nodes $i$ and $j$ that has no collider section, and
\item the route has no node in $Z \subseteq V \setminus ij$.
\end{enumerate}
Then, there exists a probability distribution $p \in \mathcal{N}(G)$ such that $i \nci_p j | Z$.
\end{lemma}

\begin{proof}
The route in the lemma can be converted into a path $\rho$ between $i$ and $j$ in $G$ as follows: Iteratively, remove from the route any subroute between a node and itself. Note that none of these removals produces a collider section: It suffices to note that if the route after the removal has a collider section, then the route before the removal must have a collider section, which is a contradiction. Consequently, $\rho$ is a path between $i$ and $j$ in $G$ that has no collider section. Therefore, $\rho$ is superactive with respect to $Z$: Since the route in the lemma has no node in $Z$, $\rho$ has no node in $Z$ either. Now, remove from $G$ all the edges that are not in $\rho$, and call the resulting CG $H$. Note that $H$ has no complex since $\rho$ has no collider section. Drop the direction of every edge in $H$ and call the resulting UG $L$. Now, note that there exists a regular Gaussian distribution $p$ that is faithful to $L$ \citep[Corollary 3]{LnenickaandMatus2007} and, thus, $i \nci_p j | Z$ because $i \nci_L j | Z$. Note also that the fact that $p$ is faithful to $L$ implies that $p$ is Markovian with respect to $L$ which, in turn, implies that $p$ is also Markovian with respect to $H$, because $H$ and $L$ have the same underlying UG and complexes \citep[Theorem 5.6]{Frydenberg1990}. Consequently, $p \in \mathcal{N}(H)$ \citep[Proposition 3.30, Theorems 3.34 and 3.36]{Lauritzen1996} and, thus, $p \in \mathcal{N}(G)$ because $\mathcal{N}(H) \subseteq \mathcal{N}(G)$ by Lemma \ref{lem:subset}.
\end{proof}

\begin{lemma}\label{lem:polynomial}
Let $G$ be a CG. For every $i, j \in V$ and $Z \subseteq V \setminus ij$, there exists a real polynomial $S(i,j,Z)$ in the nd parameters in the parameterization of the probability distributions in $\mathcal{N}(G)$ such that, for every $p \in \mathcal{N}(G)$, $i \ci_p j | Z$ iff $S(i,j,Z)$ vanishes for the nd parameter values coding $p$.
\end{lemma}

\begin{proof}
Let $\Sigma$ denote the covariance matrix of $p$. Note that $i \ci_p j | Z$ iff $((\Sigma_{ijZ,ijZ})^{-1})_{i,j}=0$ \citep[Proposition 5.2]{Lauritzen1996}. Recall that $((\Sigma_{ijZ,ijZ})^{-1})_{i,j}=(-1)^\alpha det(\Sigma_{iZ,jZ}) / det(\Sigma_{ijZ,ijZ})$ with $\alpha \in \{0,1\}$. Note that $det(\Sigma_{ijZ,ijZ}) > 0$ because $\Sigma_{ijZ,ijZ}$ is positive definite \citep[p. 237]{Studeny2005}. Then, $i \ci_p j | Z$ iff $det(\Sigma_{iZ,jZ})=0$. Thus, $i \ci_p j | Z$ iff a real polynomial $R(i,j,Z)$ in the entries of $\Sigma$ vanishes due to Remark \ref{rem:inverse}. However, note that it follows from Lemma \ref{lem:121cg} and Remark \ref{rem:polynomial} that each entry of $\Sigma$ is a fraction of real polynomials in the nd parameters in the parameterization of the probability distributions in $\mathcal{N}(G)$. Recall also from Remark \ref{rem:polynomial} that the polynomial in the denominator of each of these fractions is non-vanishing in the nd parameter space for $\mathcal{N}(G)$. Therefore, by simple algebraic manipulation, the polynomial $R(i,j,Z)$ can be expressed as a fraction $S(i,j,Z) / T(i,j,Z)$ of real polynomials in the nd parameters where $T(i,j,Z)$ is non-vanishing in the nd parameter space for $\mathcal{N}(G)$. Consequently, $i \ci_p j | Z$ iff the real polynomial $S(i,j,Z)$ in the nd parameters vanishes for the values coding $p$.
\end{proof}

We interpret the polynomial in the lemma above as a real function on a real Euclidean space that includes the nd parameter space for $\mathcal{N}(G)$. We say that the polynomial in the lemma above is non-trivial if not all the values of the nd parameters are solutions to the polynomial. This is equivalent to the requirement that the polynomial is not identically zero, because the nd parameter space for $\mathcal{N}(G)$ contains a $d$-dimensional interval in $\mathbb{R}^d$, where $d$ is the dimension of $G$ (recall the proof of Lemma \ref{lem:spacecg}).

Let $\nu$ denote an undirected route $v_{2} - \ldots - v_{l-1}$ in a CG. Hereinafter, we denote by $v_1 \rightarrow \nu \leftarrow v_l$ the route $v_1 \rightarrow v_{2} - \ldots - v_{l-1} \leftarrow v_l$.

\begin{lemma}\label{lem:chunk2}
Let $G$ be a CG such that
\begin{enumerate}
\item $G$ has a route $i \rightarrow \nu \leftarrow j$ where $i, j \in V$ and $\nu$ is an undirected route, and
\item some node in $\nu$ is in $Z$ or has a descendant in $Z$, where $Z \subseteq V \setminus ij$.
\end{enumerate}
Then, there exists a probability distribution $p \in \mathcal{N}(G)$ such that $i \nci_p j | Z$.
\end{lemma}

\begin{proof}
The route $\nu$ can be converted into a path $\vartheta$ in $G$ as follows: Iteratively, remove from $\nu$ any subroute between a node and itself. Note that $\nu$ does not contain either $i$ or $j$ because, otherwise, $G$ would have a directed pseudocycle between $i$ and itself or between $j$ and itself, which is a contradiction. Therefore, $\vartheta$ does not contain either $i$ or $j$ and, thus, $i \rightarrow \vartheta \leftarrow j$ is a path in $G$. Note that the subroutes removed from $\nu$ contain only undirected edges. Therefore, every node that is in $\nu$ but not in $\vartheta$ is a descendant of some node in $\vartheta$. Consequently, some node in $\vartheta$ is in $Z$ or has a descendant in $Z$, due to the assumptions in the lemma.

We first prove the lemma for the case where some node in $\vartheta$ is in $Z$. Remove from $G$ all the edges that are not in $i \rightarrow \vartheta \leftarrow j$, and call the resulting CG $H$. Note that $i \rightarrow \vartheta \leftarrow j$ is a complex in $H$ and, thus, that $i \ci_{H} j$. Let $k$ denote the closest node to $i$ that it is in $\vartheta$ and in $Z$.

We prove in this paragraph that there exists a probability distribution $p \in \mathcal{N}(H)$ such that $i \nci_p j Z \setminus k | k$. By Lemma \ref{lem:polynomial}, there exists a real polynomial $S(i,k,\emptyset)$ in the nd parameters in the parameterization of the probability distributions in $\mathcal{N}(H)$ such that, for every $q \in \mathcal{N}(H)$, $i \ci_q k$ iff $S(i,k,\emptyset)$ vanishes for the nd parameter values coding $q$. Furthermore, $S(i,k,\emptyset)$ is non-trivial. To see this, remove from $H$ all the edges outside the path between $i$ and $k$, and call the resulting CG $L$. Note that $\mathcal{N}(L) \subseteq \mathcal{N}(H)$ by Lemma \ref{lem:subset}. Now note that, by Lemma \ref{lem:chunk1}, there exists a probability distribution $r \in \mathcal{N}(L)$ such that $i \nci_r k$. By an analogous reasoning, we can conclude that there exists a non-trivial real polynomial $S(j,k,\emptyset)$ in the nd parameters in the parameterization of the probability distributions in $\mathcal{N}(H)$ such that, for every $q \in \mathcal{N}(H)$, $j \ci_q k$ iff $S(j,k,\emptyset)$ vanishes for the nd parameter values coding $q$. Let $sol(i,k,\emptyset)$ and $sol(j,k,\emptyset)$ denote the sets of solutions to the polynomials $S(i,k,\emptyset)$ and $S(j,k,\emptyset)$, respectively. Let $d$ denote the dimension of $H$. Then, $sol(i,k,\emptyset)$ and $sol(j,k,\emptyset)$ have both zero Lebesgue measure with respect to $\mathbb{R}^d$ because they consist of the solutions to non-trivial real polynomials in real variables (the nd parameters) \citep{Okamoto1973}. Then, $sol=sol(i,k,\emptyset) \cup sol(j,k,\emptyset)$ also has zero Lebesgue measure with respect to $\mathbb{R}^d$, because the finite union of sets of zero Lebesgue measure has zero Lebesgue measure too. Consequently, the probability distributions $q \in \mathcal{N}(H)$ such that $i \ci_q k$ or $j \ci_q k$ correspond to a set of elements of the nd parameter space for $\mathcal{N}(H)$ that has zero Lebesgue measure with respect to $\mathbb{R}^d$ because it is contained in $sol$. Since this correspondence is one-to-one by Lemma \ref{lem:121cg}, the probability distributions $q \in \mathcal{N}(H)$ such that $i \ci_q k$ or $j \ci_q k$ also have zero Lebesgue measure with respect to $\mathbb{R}^d$. This result together with Theorem \ref{the:f1} imply that there exists a probability distribution $p \in \mathcal{N}(H)$ such that $i \nci_p k$ and $j \nci_p k$. Furthermore, as shown above $i \ci_{H} j$ and, thus, $i \ci_p j$ because $p$ is Markovian with respect to $H$, since $p \in \mathcal{N}(H)$ \citep[Proposition 3.30, Theorems 3.34 and 3.36]{Lauritzen1996}. Then, $i \nci_p j | k$ by symmetry and weak transitivity and, thus, $i \nci_p j Z \setminus k | k$ by decomposition.

Finally, recall that $k$ is the closest node to $i$ that it is in $\vartheta$ and in $Z$, then $i \ci_H Z \setminus k | j k$ and thus $i \ci_p Z \setminus k | j k$ because $p$ is Markovian with respect to $H$. Then, $i \nci_p j | Z$ by intersection on $i \nci_p j Z \setminus k | k$ and $i \ci_p Z \setminus k | j k$. Consequently, we have proven that there exists a probability distribution $p \in \mathcal{N}(H)$ such that $i \nci_p j | Z$. Moreover, $p \in \mathcal{N}(G)$ because $\mathcal{N}(H) \subseteq \mathcal{N}(G)$ by Lemma \ref{lem:subset}.

We now prove the lemma for the case where no node in $\vartheta$ is in $Z$ but some node in $\vartheta$ has a descendant in $Z$. Consider the shortest descending path between a node in $\vartheta$ and a node in $Z$. Let $l$ and $k$ denote the initial and final nodes of the path, i.e. $k \in Z$. Remove from $G$ all the edges that are not in $i \rightarrow \vartheta \leftarrow j$ or in the path between $l$ and $k$, and call the resulting CG $H$. Note that $i \rightarrow \vartheta \leftarrow j$ is a complex in $H$ and, thus, that $i \ci_{H} j$. Therefore, we can follow the same steps as above to prove that there exists a probability distribution $p \in \mathcal{N}(H)$ such that $i \nci_p j Z \setminus k | k$. Finally, recall that there is no path between $i$ and any node in $Z \setminus k$ in $H$, then $i \ci_H Z \setminus k | j k$ and thus $i \ci_p Z \setminus k | j k$ because $p$ is Markovian with respect to $H$. Then, $i \nci_p j | Z$ by intersection on $i \nci_p j Z \setminus k | k$ and $i \ci_p Z \setminus k | j k$. Consequently, we have proven that there exists a probability distribution $p \in \mathcal{N}(H)$ such that $i \nci_p j | Z$. Moreover, $p \in \mathcal{N}(G)$ because $\mathcal{N}(H) \subseteq \mathcal{N}(G)$ by Lemma \ref{lem:subset}.
\end{proof}

\begin{lemma}\label{lem:non-trivial}
Let $G$ be a CG such that $i \nci_G j | Z$, where $i, j \in V$ and $Z \subseteq V \setminus ij$. Then, there exists a probability distribution $p \in \mathcal{N}(G)$ such that $i \nci_p j | Z$.
\end{lemma}

\begin{proof}
We prove the lemma in two steps. In the first step, we introduce some notation that we use in the second step, the actual proof of the lemma.

\textbf{Step 1} Given a route $\rho$ in a CG $H$, we define $H_{\rho}$ as the CG resulting from removing from $H$ all the edges that are not in $\rho$. We define the level of a node in $H$ as the index of the connectivity component the node belongs to. We define the dlength of a route as the number of distinct edges in the route. Note the difference between the dlength and the length of a route: The former counts edges without repetition and the latter with repetition (recall Section \ref{sec:preliminaries}). We say that a route is dshorter than another route if the former has smaller dlength than the latter. Likewise, we say that a route is dshortest if no other route is dshorter than it. Let $\mathfrak{N}$ denote any total order of the nodes in the CG $H$. Let $\mathfrak{R}$ denote any total order of all the routes between two nodes in $H$. Finally, if $a \nci_H b | C$ where $a, b \in V$ and $C \subseteq V \setminus ab$, then we define $splits(a,b,C,H)$ as follows:
\begin{itemize}
\item[S1.] If there is a route in $H$ like that in Lemma \ref{lem:chunk1} or \ref{lem:chunk2} for $i=a$, $j=b$ and $Z=C$, then we define $splits(a,b,C,H)=0$.
\item[S2.] Otherwise, we define recursively $splits(a,b,C,H)=splits(a,k,C,H_{\rho})+splits(b,k,C,H_{\rho})+1$, where $\rho$ and $k$ are selected as follows. Let $\Psi$ denote the set of routes between $a$ and $b$ in $H$ that are superactive with respect to $C$. Let $\Phi$ denote the dshortest routes in $\Psi$. Let $\Upsilon$ denote the shortest routes in $\Phi$. Let $\rho$ denote the route in $\Upsilon$ that comes first in $\mathfrak{R}$. We call $\rho$ the splitting route. Furthermore, let $K$ denote the set of nodes in $\rho$ but not in $C a b$ that have minimal level in $H_{\rho}$. Let $k$ denote the node in $K$ that comes first in $\mathfrak{N}$. Note that the only point with $\mathfrak{R}$ and $\mathfrak{N}$ is to select $\rho$ and $k$ unambiguously.
\end{itemize}

Note that we have implicitly assumed in the definition S2 that $K$ is non-empty. We now prove that this is always true. Assume to the contrary that $K$ is empty. This means that all the nodes in $\rho$ are in $C a b$. Since the definition S1 did not apply, $\rho$ must have some collider section $\nu$. Moreover, $a=v_1 \rightarrow \nu \leftarrow v_l=b$ is a subroute of $\rho$: If $v_1 \neq a, b$ (resp. $v_l \neq a, b$) then $v_1$ (resp. $v_l$) must be outside $C$ for $\rho$ to be superactive with respect to $C$, which contradicts the assumption that all the nodes in $\rho$ are in $C a b$. Moreover, some node in $\nu$ must be in $C$ for $\rho$ to be superactive with respect to $C$. However, this implies that $a \rightarrow \nu \leftarrow b$ is a route that satisfies the requirements of the definition S1, which is a contradiction.

Finally, we prove that $splits(a,k,C,H_{\rho})$ and $splits(b,k,C,H_{\rho})$ in the definition S2 are well-defined. Let $\varrho$ denote the subroute of $\rho$ between the first occurrences of $a$ and $k$ in $\rho$ when going from $a$ to $b$. Note that if $\rho$ contains $k$ only in non-collider sections, then none of the other nodes in those sections can be in $C$ for $\rho$ to be superactive with respect to $C$ and, thus, $\varrho$ is a route between $a$ and $k$ in $H_{\rho}$ that is superactive with respect to $C$ and, thus, $a \nci_{H_{\rho}} k | C$ and, thus, $splits(a,k,C,H_{\rho})$ is defined. We now prove that $\rho$ contains $k$ only in non-collider sections. Assume the contrary and let $\nu$ denote any collider section of $\rho$ that contains $k$. Note that $a=v_1 \rightarrow \nu \leftarrow v_l=b$ is a subroute of $\rho$, because if $v_1 \neq a, b$ or $v_l \neq a, b$ then there exists a node in $\rho$ but not in $C a b$ with smaller level than $k$ in $H_{\rho}$, which is a contradiction. Moreover, some node in $\nu$ must be in $C$ for $\rho$ to be superactive with respect to $C$. However, this implies that $a \rightarrow \nu \leftarrow b$ is a route that satisfies the requirements of the definition S1, which is a contradiction. Now, let $\varphi$ denote the subroute of $\rho$ between the first occurrences of $b$ and $k$ in $\rho$ when going from $b$ to $a$. By repeating the reasoning above with $\varphi$ instead of $\varrho$, we can conclude that $b \nci_{H_{\rho}} k | C$ and, thus, that $splits(b,k,C,H_{\rho})$ is defined too. Moreover, note that $\varrho$ and $\varphi$ have dlength equal or smaller than $\rho$ and length strictly smaller than $\rho$. Therefore, the splitting routes for $splits(a,k,C,H_{\rho})$ and $splits(b,k,C,H_{\rho})$ are each either dshorter or shorter than $\rho$. This guarantees that the recursive definition S2 eventually reaches the trivial case S1.

\textbf{Step 2} We prove the lemma by induction over the value of $splits(i,j,Z,G)$. If $splits(i,j,Z,G)=0$, then there exists a route in $G$ like that in Lemma \ref{lem:chunk1} or \ref{lem:chunk2}. Therefore, there exists a probability distribution $p \in \mathcal{N}(G)$ such that $i \nci_p j | Z$ by Lemma \ref{lem:chunk1} or \ref{lem:chunk2}.

Assume as induction hypothesis that the lemma holds for any value of $splits(i,j,Z,G)$ smaller than $m$ ($m > 0$). We now prove it for value $m$. Recall that $splits(i,j,Z,G)=splits(i,k,Z,G_{\rho})+splits(j,k,Z,G_{\rho})+1$ where $\rho$ is a dshortest route among all the routes between $i$ and $j$ in $G$ that are superactive with respect to $Z$, and $k$ is a node in $\rho$ but not in $Z i j$ that has minimal level in $G_{\rho}$. Then, as shown in Step 1, $i \nci_{G_{\rho}} k | Z$ and $j \nci_{G_{\rho}} k | Z$. Moreover, $splits(i,k,Z,G_{\rho})$ and $splits(j,k,Z,G_{\rho})$ are both smaller than $m$. Then, by the induction hypothesis, there exist two probability distributions $r, s \in \mathcal{N}(G_{\rho})$ such that $i \nci_r k | Z$ and $j \nci_s k | Z$. We prove below that there exists a probability distribution $p \in \mathcal{N}(G_{\rho})$ such that $i \nci_p j | Z$. Note that $p \in \mathcal{N}(G)$ because $\mathcal{N}(G_{\rho}) \subseteq \mathcal{N}(G)$ by Lemma \ref{lem:subset}.

By Lemma \ref{lem:polynomial}, there exists a real polynomial $S(i,k,Z)$ in the nd parameters in the parameterization of the probability distributions in $\mathcal{N}(G_{\rho})$ such that, for every $q \in \mathcal{N}(G_{\rho})$, $i \ci_q k | Z$ iff $S(i,k,Z)$ vanishes for the nd parameter values coding $q$. Furthermore, $S(i,k,Z)$ is non-trivial due to the probability distribution $r$ above. Similarly, there exists a real polynomial $S(j,k,Z)$ in the nd parameters in the parameterization of the probability distributions in $\mathcal{N}(G_{\rho})$ such that, for every $q \in \mathcal{N}(G_{\rho})$, $j \ci_q k | Z$ iff $S(j,k,Z)$ vanishes for the nd parameter values coding $q$. Furthermore, $S(j,k,Z)$ is also non-trivial due to the probability distribution $s$ above. Let $sol(i,k,Z)$ and $sol(j,k,Z)$ denote the sets of solutions to the polynomials $S(i,k,Z)$ and $S(j,k,Z)$, respectively. Let $d$ denote the dimension of $G_{\rho}$. Then, $sol(i,k,Z)$ and $sol(j,k,Z)$ have both zero Lebesgue measure with respect to $\mathbb{R}^d$ because they consist of the solutions to non-trivial real polynomials in real variables (the nd parameters) \citep{Okamoto1973}. Then, $sol=sol(i,k,Z) \cup sol(j,k,Z)$ also has zero Lebesgue measure with respect to $\mathbb{R}^d$, because the finite union of sets of zero Lebesgue measure has zero Lebesgue measure too. Consequently, the probability distributions $q \in \mathcal{N}(G_{\rho})$ such that $i \ci_q k | Z$ or $j \ci_q k | Z$ correspond to a set of elements of the nd parameter space for $\mathcal{N}(G_{\rho})$ that has zero Lebesgue measure with respect to $\mathbb{R}^d$ because it is contained in $sol$. Since this correspondence is one-to-one by Lemma \ref{lem:121cg}, the probability distributions $q \in \mathcal{N}(G_{\rho})$ such that $i \ci_q k | Z$ or $j \ci_q k | Z$ also have zero Lebesgue measure with respect to $\mathbb{R}^d$. This result together with Theorem \ref{the:f1} imply that there exists a probability distribution $p \in \mathcal{N}(G_{\rho})$ such that $i \nci_p k | Z$ and $j \nci_p k | Z$. Note that these two independence statements together with $i \ci_p j | Z k$ would imply the desired result by symmetry and weak transitivity. We prove below $i \ci_{G_{\rho}} j | Z k$ which, in turn, implies $i \ci_p j | Z k$ because $p$ is Markovian with respect to $G_{\rho}$, since $p \in \mathcal{N}(G_{\rho})$ \citep[Proposition 3.30, Theorems 3.34 and 3.36]{Lauritzen1996}.

Assume to the contrary $i \nci_{G_{\rho}} j | Z k$. Let $\varrho$ denote any route between $i$ and $j$ in $G_{\rho}$ that is superactive with respect to $Z k$. Note that $\varrho$ must contain $k$ because, otherwise, $\varrho$ would be a route between $i$ and $j$ in $G$ that is superactive with respect to $Z$ and that is dshorter than $\rho$, which is a contradiction. Furthermore, $\varrho$ must contain $k$ only in collider sections because, otherwise, $\varrho$ would not be superactive with respect to $Z k$. Let $\nu$ denote any collider section of $\varrho$ that contains $k$. Note that $i=v_1 \rightarrow \nu \leftarrow v_l=j$ is a subroute of $\varrho$, because if $v_1 \neq i, j$ or $v_l \neq i, j$ then there exists a node in $\varrho$ but not in $Z i j$ with smaller level than $k$ in $G_{\rho}$. Since $\varrho$ is a route in $G_{\rho}$, this implies that there exists a node in $\rho$ but not in $Z i j$ with smaller level than $k$ in $G_{\rho}$, which is a contradiction. Note also that no descendant of $k$ in $G$ can be in $Z$ because, otherwise, $i \rightarrow \nu \leftarrow j$ would be a route that satisfies the requirements of the definition S1, which is a contradiction. However, if no descendant of $k$ in $G$ is in $Z$, then $\rho$ must contain $k$ only in non-collider sections because, otherwise, $\rho$ would not be superactive with respect to $Z$. The last two observations imply that $i$ or $j$ is a descendant of $k$ in $G$ which, together with $i \rightarrow \nu \leftarrow j$, implies that $G$ has a directed pseudocycle between $i$ and itself or between $j$ and itself, because $\nu$ contains $k$. This is a contradiction.
\end{proof}

\begin{theorem}\label{the:f2}
Let $G$ be a CG of dimension $d$. The set of probability distributions in $\mathcal{N}(G)$ that are not faithful to $G$ has zero Lebesgue measure with respect to $\mathbb{R}^d$.
\end{theorem}

\begin{proof}
Note that the probability distributions in $\mathcal{N}(G)$ are Markovian with respect to $G$ \citep[Proposition 3.30, Theorems 3.34 and 3.36]{Lauritzen1996}. Then, for any probability distribution $p \in \mathcal{N}(G)$ not to be faithful to $G$, $p$ must satisfy some independence that is not entailed by $G$. That is, there must exist three disjoint subsets of $V$, here denoted as $I$, $J$ and $Z$, such that $I \nci_G J | Z$ but $I \ci_p J | Z$. However, if $I \nci_G J | Z$ then $i \nci_G j | Z$ for some $i \in I$ and $j \in J$. Furthermore, if $I \ci_p J | Z$ then $i \ci_p j | Z$ by symmetry and decomposition. By Lemma \ref{lem:polynomial}, there exists a real polynomial $S(i,j,Z)$ in the nd parameters in the parameterization of the probability distributions in $\mathcal{N}(G)$ such that, for every $q \in \mathcal{N}(G)$, $i \ci_q j | Z$ iff $S(i,j,Z)$ vanishes for the nd parameter values coding $q$. Furthermore, $S(i,j,Z)$ is non-trivial by Lemma \ref{lem:non-trivial}. Let $sol(i, j, Z)$ denote the set of solutions to the polynomial $S(i,j,Z)$. Then, $sol(i, j, Z)$ has zero Lebesgue measure with respect to $\mathbb{R}^d$ because it consists of the solutions to a non-trivial real polynomial in real variables (the nd parameters) \citep{Okamoto1973}. Then, $sol= \bigcup_{\{I, J, Z \subseteq V \: \mbox{\scriptsize disjoint} \: : \: I \nci_G J | Z\}} \bigcup_{\{i \in I, j \in J \: : \: i \nci_G j | Z\}} sol(i, j, Z)$ has zero Lebesgue measure with respect to $\mathbb{R}^d$, because the finite union of sets of zero Lebesgue measure has zero Lebesgue measure too. Consequently, the probability distributions in $\mathcal{N}(G)$ that are not faithful to $G$ correspond to a set of elements of the nd parameter space for $\mathcal{N}(G)$ that has zero Lebesgue measure with respect to $\mathbb{R}^d$ because it is contained in $sol$. Since this correspondence is one-to-one by Lemma \ref{lem:121cg}, the probability distributions in $\mathcal{N}(G)$ that are not faithful to $G$ also have zero Lebesgue measure with respect to $\mathbb{R}^d$.
\end{proof}

The following corollary, which follows trivially from Theorems \ref{the:f1} and \ref{the:f2}, summarizes the results in this section.

\begin{corollary}\label{co:f3}
Let $G$ be a CG of dimension $d$. The set of probability distributions in $\mathcal{N}(G)$ that are faithful to $G$ has positive Lebesgue measure with respect to $\mathbb{R}^d$.
\end{corollary}

\section{Equivalence in chain graphs}\label{sec:equivalence}

The space of CGs can be divided in classes of equivalent CGs according to criteria such as Markov independence equivalence, Markovian distribution equivalence or factorization equivalence. As we prove below with the help of the theorems above, these criteria actually coincide. This result is important because the classes of Markovian distribution equivalent CGs have a simple graphical characterization and a natural representative, the so-called largest CG, which now also apply to the classes of equivalence induced by the other two criteria mentioned. We also prove below that all equivalent CGs have the same dimension with respect to the parameterization introduced in Section \ref{sec:faithfulness}.

Before proving our results, we formally define the equivalence criteria discussed in the paragraph above. Recall that, unless otherwise stated, all the probability distributions in this paper are defined on (state space) $\mathbb{R}^N$, where $|V|=N$. We say that two CGs are Markov independence equivalent if they represent the same independence model. We say that two CGs are Markovian distribution equivalent if every regular Gaussian distribution is Markovian with respect to both CGs or with respect to neither of them. We say that two CGs $G$ and $H$ are factorization equivalent if $\mathcal{N}(G)=\mathcal{N}(H)$. The corollary below proves that these definitions coincide.

\begin{corollary}\label{co:me}
Let $G$ and $H$ denote two CGs. The following statements are equivalent in the frame of regular Gaussian distributions:
\begin{enumerate}
\item $G$ and $H$ are factorization equivalent.
\item $G$ and $H$ are Markovian distribution equivalent.
\item $G$ and $H$ are Markov independence equivalent.
\end{enumerate}
\end{corollary}

\begin{proof}
The equivalence of Statements 1 and 2 follows from \citep[Proposition 3.30, Theorems 3.34 and 3.36]{Lauritzen1996}. We now prove that Statements 2 and 3 are equivalent. By definition, Markov independence equivalence implies Markovian distribution equivalence. To see the opposite implication, note that if $G$ and $H$ are not Markov independence equivalent, then one of them, say $G$, must represent a separation statement $I \ci_G J | K$ that is not represented by $H$. Consider a probability distribution $p \in \mathcal{N}(H)$ faithful to $H$. Such a probability distribution exists due to Corollary \ref{co:f3}, and it is Markovian with respect to $H$. However, $p$ cannot be Markovian with respect to $G$, because $I \nci_H J | K$ implies $I \nci_p J | K$.
\end{proof}

\citet[Theorem 5.6]{Frydenberg1990} gives a straightforward graphical characterization of Markovian distribution equivalence: Two CGs are Markovian distribution equivalent iff they have the same underlying UG and the same complexes. Due to the corollary above, that is also a graphical characterization of the other two types of equivalence discussed there. Hereinafter, we do not distinguish anymore between the different types of equivalence discussed in the corollary above because they coincide and, thus, we simply refer to them as equivalence.

\citet[Proposition 5.7]{Frydenberg1990} shows that every class of equivalent CGs contains a unique CG that has more undirected edges than any other CG in the class. Such a CG is called the largest CG (LCG) in the class, and it is usually considered a natural representative of the class. \citet[Section 4.2]{Studeny1998} conjectures that, for discrete probability distributions, the LCG in a class of equivalent CGs has fewer nd parameters than any other CG in the class. This would imply that the most space efficient way of storing the discrete probability distributions that factorize with respect to a class of equivalent CGs is by factorizing them with respect to the LCG in the class rather than with respect to any other CG in the class. The corollary below proves that an analogous conjecture for regular Gaussian distributions and the parameterization of them proposed in Section \ref{sec:faithfulness} would be false.

\begin{corollary}\label{co:lcg}
All equivalent CGs have the same dimension with respect to the parameterization proposed in Section \ref{sec:faithfulness}.
\end{corollary}

\begin{proof}
Let $G$ denote the LCG in a class of equivalent CGs. Let $H$ denote any other CG in the class. Recall that the dimensions of $G$ and $H$ with respect to the parameterization proposed in Section \ref{sec:faithfulness} are, respectively, $2|V|+|G|$ and $2|V|+|H|$. Note that $H$ can be obtained from $G$ by orienting some of the undirected edges in $G$ \citep[Theorem 3.9]{VolfandStudeny1999}. Then, $|H|=|G|$ and, thus, $2|V|+|G|=2|V|+|H|$.
\end{proof}

\section{Conclusions}\label{sec:conclusions}

In this paper, we have proven that, in certain measure-theoretic sense, almost all the regular Gaussian distributions that factorize with respect to a chain graph are faithful to it. This result extends previous results such as
\begin{itemize}
\item \citep[Theorem 3.2]{Spirtesetal.1993} where it is proven that, in certain measure-theoretic sense, almost all the regular Gaussian distributions that factorize with respect to an acyclic directed graph are faithful to it, and
\item \citep[Corollary 3]{LnenickaandMatus2007} where it is proven that for any undirected graph there exists a regular Gaussian distribution that is faithful to it.
\end{itemize}

There are a number of consequences that follow from the result proven in this paper:
\begin{itemize}
\item There are independence models that can be represented exactly by chain graphs but that cannot be represented exactly by undirected graphs or acyclic directed graphs. As a matter of fact, the experimental results in \citep{Penna2007} suggest that this may be the case for the vast majority of independence models that can be represented exactly by chain graphs. This is an advantage of chain graphs when dealing with regular Gaussian distributions, because there exists a regular Gaussian distribution that is faithful to each of these independence models.
\item The moralization and c-separation criteria for reading independencies holding in the regular Gaussian distributions that factorize with respect to a chain graph are complete (i.e. they identify all the independencies that can be identified on the sole basis of the chain graph), because there exists a regular Gaussian distribution that is faithful to the chain graph.
\item Some definitions of equivalence in chain graphs coincide, which implies that the graphical characterization of Markovian distribution equivalence in \citep[Theorem 5.6]{Frydenberg1990} also applies to other definitions of equivalence.
\item For the parameterization introduced in this paper, all the chain graphs in a class of equivalence have the same dimension and, thus, their factorizations are equally space efficient for storing the regular Gaussian distribution that factorize with respect to the chain graphs in the class.
\end{itemize}

\section*{Acknowledgements}

This work is funded by the Swedish Research Council (ref. VR-621-2005-4202) and CENIIT at Link\"oping University (ref. 09.01). We are thankful to Prof. Jose A. Lozano for discussions on this work.

\section*{Appendix A}

In this appendix, we derive Equations \ref{eq:bishop1a}-\ref{eq:bishop2c}. Our derivations are adaptations of those in \citep[Sections 2.3.1, 2.3.3]{Bishop2006} to the notation used in this paper. A Gaussian distribution for $X$ can be written as
\[
p(x) = \mathcal{N}(\mu,\Omega^{-1})=\frac{e^{- \frac{1}{2}(x-\mu)^T \Omega (x-\mu)}}{k}
\]
where $\mu$ is a $|V|$-dimensional mean vector, $\Omega^{-1}$ a $|V| \times |V|$-dimensional covariance matrix, and $k$ a normalization constant. Note that
\[
- \frac{1}{2}(x-\mu)^T \Omega (x-\mu)
= - \frac{1}{2} [ x^T \Omega x - x^T \Omega \mu - \mu^T \Omega x + \mu^T \Omega \mu ]
\]
\[
= - \frac{1}{2} x^T \Omega x + x^T \Omega \mu + k'
\]
where $k'$ is a constant, i.e. it is independent of $x$. In the last equality above we have used the fact that
\begin{equation}\label{eq:ap1}
x^T \Omega \mu = (x^T \Omega \mu )^T = \mu^T ( x^T \Omega )^T = \mu^T (\Omega)^T (x^T)^T = \mu^T \Omega x
\end{equation}
because $\Omega$ is symmetric. Then, a Gaussian distribution for $X$ can be written as
\begin{equation}\label{eq:ap2}
p(x) = \mathcal{N}(\mu,\Omega^{-1})=\frac{e^{- \frac{1}{2} x^T \Omega x + x^T \Omega \mu}}{k''}
\end{equation}
where $k''$ is a normalization constant.

Let $I$ and $J$ denote two disjoint subsets of $V$. Let $p(x_{IJ}) = \mathcal{N}(\mu, \Omega^{-1})$ where $\Omega$ is positive definite. If we regard $x_I$ as a constant, then
\[
- \frac{1}{2}(x-\mu)^T \Omega (x-\mu)
= - \frac{1}{2} [ (x_I-\mu_I)^T \Omega_{I,I} (x_I-\mu_I) + (x_I-\mu_I)^T \Omega_{I,J} (x_J-\mu_J)
\]
\[
+ (x_J-\mu_J)^T \Omega_{J,I} (x_I-\mu_I) + (x_J-\mu_J)^T \Omega_{J,J} (x_J-\mu_J) ]
\]
\[
= - \frac{1}{2} [ x_I^T \Omega_{I,J} x_J - \mu_I^T \Omega_{I,J} x_J + x_J^T \Omega_{J,I} x_I - x_J^T \Omega_{J,I} \mu_I + x_J^T \Omega_{J,J} x_J - x_J^T \Omega_{J,J} \mu_J
\]
\[
- \mu_J^T \Omega_{J,J} x_J ] + k''' = - \frac{1}{2} x_J^T \Omega_{J,J} x_J + x_J^T [ \Omega_{J,J} \mu_J - \Omega_{J,I} (x_I-\mu_I)] + k'''
\]
where $k'''$ is a constant, i.e. it is independent of $x_J$. In the last equality above we have used a reasoning analogous to that in Equation \ref{eq:ap1}. Then,
\begin{equation}\label{eq:ap3}
p(x_J | x_I) = \frac{p(x_{IJ})}{p(x_I)} = \frac{e^{- \frac{1}{2} x_J^T \Omega_{J,J} x_J + x_J^T [ \Omega_{J,J} \mu_J - \Omega_{J,I} (x_I-\mu_I) ]}}{k''''}
\end{equation}
where $k''''$ is a normalization constant, because $x_I$ can be regarded as a constant in $p(x_J | x_I)$ since it is the value of the conditioning set. Consequently, $p(x_J | x_I)$ is a Gaussian distribution since it can be written in the form given in Equation \ref{eq:ap2}. By equating the term that is quadratic in $X$ in Equation \ref{eq:ap2} with the term that is quadratic in $X_J$ in Equation \ref{eq:ap3}, we conclude that the covariance matrix of $p(x_J | x_I)$ is $(\Omega_{J,J})^{-1}$. By equating the term that is linear in $X$ in Equation \ref{eq:ap2} with the term that is linear in $X_J$ in Equation \ref{eq:ap3}, we conclude that the mean vector of $p(x_J | x_I)$ is
\[
(\Omega_{J,J})^{-1} [\Omega_{J,J} \mu_J - \Omega_{J,I} (x_I-\mu_I)] = \mu_J - (\Omega_{J,J})^{-1} \Omega_{J,I} (x_I-\mu_I)
\]
\[
= -(\Omega_{J,J})^{-1} \Omega_{J,I} x_I +  \mu_J + (\Omega_{J,J})^{-1} \Omega_{J,I} \mu_I.
\]
Therefore, $p(x_J | x_I) = \mathcal{N}(\delta x_I + \gamma, \epsilon^{-1})$ where $\delta$, $\gamma$ and $\epsilon$ are the following real matrices of dimensions, respectively, $|J| \times |I|$, $|J| \times 1$ and $|J| \times |J|$:
\[
\delta=-(\Omega_{J,J})^{-1} \Omega_{J,I},
\]
\[
\gamma=\mu_J+(\Omega_{J,J})^{-1} \Omega_{J,I} \mu_I
\]
and
\[
\epsilon=\Omega_{J,J}.
\]

Now, let $p(x_I) = \mathcal{N}(\alpha, \beta^{-1})$ and $q(x_J | x_I) = \mathcal{N}(\delta x_I + \gamma, \epsilon^{-1})$ where $\delta$, $\gamma$ and $\epsilon$ are real matrices of dimensions, respectively, $|J| \times |I|$, $|J| \times 1$ and $|J| \times |J|$, and $\beta$ and $\epsilon$ are positive definite. Then,
\[
q(x_J | x_I)p(x_I)= \frac{e^{- \frac{1}{2} [(x_J - \delta x_I - \gamma)^T \epsilon (x_J - \delta x_I - \gamma) + (x_I - \alpha)^T \beta (x_I - \alpha)]}}{k}
\]
where $k$ is a normalization constant. Note that
\[
- \frac{1}{2} [(x_J - \delta x_I - \gamma)^T \epsilon (x_J - \delta x_I - \gamma) + (x_I - \alpha)^T \beta (x_I - \alpha)]
\]
\[
= - \frac{1}{2} [ x_J^T \epsilon x_J - x_J^T \epsilon \delta x_I - x_J^T \epsilon \gamma - (\delta x_I)^T \epsilon x_J + (\delta x_I)^T \epsilon \delta x_I + (\delta x_I)^T \epsilon \gamma
\]
\[
- \gamma^T \epsilon x_J + \gamma^T \epsilon \delta x_I + x_I^T \beta x_I - x_I^T \beta \alpha - \alpha^T \beta x_I  ] + k'
\]
\[
= - \frac{1}{2} [ x_J^T \epsilon x_J - x_J^T \epsilon \delta x_I - (\delta x_I)^T \epsilon x_J + (\delta x_I)^T \epsilon \delta x_I + x_I^T \beta x_I ]
\]
\[
+ x_J^T \epsilon \gamma + x_I^T \beta \alpha - x_I^T \delta^T \epsilon \gamma + k'
\]
where $k'$ is a constant, i.e. it is independent of $x_{IJ}$. In the last equality above we have used a reasoning analogous to that in Equation \ref{eq:ap1}. By using this reasoning further and reorganizing some terms we can rewrite the expression above as
\[
- \frac{1}{2} [ x_I^T (\beta + \delta^T \epsilon \delta) x_I + x_J^T \epsilon x_J  - x_J^T \epsilon \delta x_I - x_I^T \delta^T \epsilon x_J ]
\]
\[
+ x_J^T \epsilon \gamma + x_I^T (\beta \alpha - \delta^T \epsilon \gamma) + k'
\]
\[
= - \frac{1}{2}
\left(
  \begin{array}{c}
    x_I \\
    x_J\\
  \end{array}
\right)^T
\left(
  \begin{array}{c c}
    \beta + \delta^T \epsilon \delta &  -\delta^T \epsilon \\
    - \epsilon \delta & \epsilon\\
  \end{array}
\right)
\left(
  \begin{array}{c}
    x_I \\
    x_J\\
  \end{array}
\right)
+
\left(
  \begin{array}{c}
    x_I \\
    x_J\\
  \end{array}
\right)^T
\left(
  \begin{array}{c}
    \beta \alpha - \delta^T \epsilon \gamma \\
    \epsilon \gamma\\
  \end{array}
\right)
+ k'.
\]
Then, $q(x_J | x_I)p(x_I)$ can be expressed as
\begin{equation}\label{eq:ap4}
\frac{e^{
- \frac{1}{2}
\left(
  \begin{array}{c}
    x_I \\
    x_J\\
  \end{array}
\right)^T
\left(
  \begin{array}{c c}
    \beta + \delta^T \epsilon \delta &  -\delta^T \epsilon \\
    - \epsilon \delta & \epsilon\\
  \end{array}
\right)
\left(
  \begin{array}{c}
    x_I \\
    x_J\\
  \end{array}
\right)
+
\left(
  \begin{array}{c}
    x_I \\
    x_J\\
  \end{array}
\right)^T
\left(
  \begin{array}{c}
    \beta \alpha - \delta^T \epsilon \gamma \\
    \epsilon \gamma\\
  \end{array}
\right)
}}{k''}
\end{equation}
where $k''$ is a normalization constant. Consequently, $q(x_J | x_I)p(x_I)$ is a Gaussian distribution over $\left(
\begin{array}{c}
    x_I\\
    x_J\\
  \end{array}
\right)$ since it can be expressed in the form given in Equation \ref{eq:ap2}. As we did above, the precision matrix (resp. the mean vector) of $q(x_J | x_I)p(x_I)$ can easily be found by equating the term that is quadratic (resp. linear) in $X$ in Equation \ref{eq:ap2} with the term that is quadratic (resp. linear) in $\left(
\begin{array}{c}
    x_I\\
    x_J\\
  \end{array}
\right)$ in Equation \ref{eq:ap4}. Specifically, $q(x_J | x_I)p(x_I)=\mathcal{N}(\lambda, \Lambda^{-1})$ where
\[
\Lambda = \left(
  \begin{array}{cc}
    \beta + \delta^T \epsilon \delta & - \delta^T \epsilon \\
    -\epsilon \delta & \epsilon \\
  \end{array}
\right)
\]
and
\[
\lambda=\Lambda^{-1}
\left(
  \begin{array}{c}
    \beta \alpha - \delta^T \epsilon \gamma \\
    \epsilon \gamma\\
  \end{array}
\right)
=
\left(
  \begin{array}{cc}
    \beta^{-1} & \beta^{-1} \delta^T \\
    \delta \beta^{-1} & \epsilon^{-1} + \delta \beta^{-1} \delta^T \\
  \end{array}
\right)
\left(
  \begin{array}{c}
    \beta \alpha - \delta^T \epsilon \gamma \\
    \epsilon \gamma\\
  \end{array}
\right)
\]
\[
=
\left(
  \begin{array}{c}
    \alpha \\
    \delta \alpha + \gamma \\
  \end{array}
\right).
\]
We have omitted the details of the derivation of $\Lambda^{-1}$ from $\Lambda$, but it can easily be checked that $\Lambda \Lambda^{-1}$ equals the identity matrix. Note that $\Lambda$ is invertible because $\beta$ and $\epsilon$ are invertible and, thus, that $q(x_J | x_I)p(x_I)$ is regular.

\end{document}